\documentclass[11pt]{article}

\usepackage{arxiv}
\usepackage{orcidlink}
\usepackage{graphicx}
\usepackage{amsmath,amssymb,amsfonts}
\usepackage{amsthm}
\usepackage{xcolor}
\usepackage{multirow}
\usepackage{booktabs}
\usepackage{array}
\usepackage{placeins}
\usepackage{tikz}
\usepackage{threeparttable}
\usepackage{enumitem}

\usepackage[sort&compress]{natbib}
\usepackage{hyperref}
\hypersetup{
  colorlinks=true,
  linkcolor=blue,
  citecolor=blue,
  urlcolor=magenta
}
\usepackage[nameinlink,noabbrev]{cleveref}

\DeclareMathOperator{\E}{\operatorname{E}}
\theoremstyle{plain}
\newtheorem{theorem}{Theorem}
\newtheorem{proposition}{Proposition}
\newtheorem{lemma}{Lemma}
\newtheorem{corollary}{Corollary}
\theoremstyle{remark}
\newtheorem{remark}{Remark}
\theoremstyle{definition}
\newtheorem{assumption}{Assumption}

\newcommand{\Prb}{\Pr}
\newcommand{\I}{\mathcal{I}}
\newcommand{\J}{\mathcal{J}}
\newcommand{\dd}{\,\mathrm{d}}
\newcommand{\logit}{\operatorname{logit}}
\newcommand{\tr}{\operatorname{tr}}
\newcommand{\diag}{\operatorname{diag}}
\newcommand{\ARE}{\operatorname{ARE}}
\newcommand{\RE}{\operatorname{RE}}
\newcommand{\err}{\operatorname{err}}

\title{Semi-Supervised Classification with Informative Missing Labels in Weibull Mixture Models}
\renewcommand{\shorttitle}{Weibull Mixture SSL}
\renewcommand{\headeright}{Wu, Wang, and McLachlan}

\author{
Jinran Wu~\orcidlink{0000-0002-2388-3614}\\
The University of Queensland, Australia\\
\texttt{jinran.wu@uq.edu.au}
\And
You-Gan Wang~\orcidlink{0000-0003-0901-4671}\\
The University of Queensland, Australia\\
\texttt{you-gan.wang@uq.edu.au}
\And
Geoffrey J. McLachlan~\orcidlink{0000-0002-5921-3145}\\
The University of Queensland, Australia\\
\texttt{g.mclachlan@uq.edu.au}
}
\date{}

\begin{document}
\maketitle

\begin{abstract}
We consider semi-supervised classification from a partially classified sample arising from a two-component Weibull mixture. The feature is observed for all data, whereas some class labels are missing. The probability of a missing label is modelled as a function of classification uncertainty, giving a feature-dependent missing-at-random (MAR) mechanism that shares parameters with the Weibull-mixture classifier. The missing-label indicators can therefore provide information about the classifier in addition to the observed features and available class labels. Under a common Weibull shape, a Bayes' rule has at most one positive decision boundary, which is unique when the rule is nonconstant; under unequal shapes, it can have two. We characterise these decision regions, derive the Fisher information for the classifier after adjustment for nuisance parameters in the missingness model, and obtain a decision-boundary expansion of the expected error rate of the plug-in sample rule relative to the Bayes error. The expansion yields classification-specific asymptotic relative efficiency formulas for the one- and two-boundary cases and shows that a positive-definite increase in Fisher information is sufficient, but not necessary, for a smaller first-order expected error rate. Numerical studies and a semi-synthetic analysis based on hard-drive failure data illustrate potential reductions in expected error rate and improvements in decision-boundary estimation from modelling feature-dependent label missingness.
\end{abstract}


\section{Introduction}\label{sec:introduction}

We consider semi-supervised classification for positive-valued data modelled by a two-component Weibull mixture. The Weibull distribution is widely used for lifetime and reliability data because its shape parameter accommodates a broad range of failure-rate behaviour \citep{menon1963,lai2011,zhao2011assessing}. Weibull mixtures allow for latent heterogeneity and multiple failure mechanisms and have a substantial statistical literature \citep{woodward1987,jewell1982, jiangmurthy1998, bucar2004,tsionas2002,farcomeni2010,panteleeva2015}. In a classification setting, \citet{ahmad1994updating} studied a nonlinear discriminant function for a two-component Weibull mixture using classified and unclassified observations. Here our interest is in the additional structure that arises when class labels are missing in a way that depends on the observed feature.

In likelihood-based semi-supervised learning, unlabelled data can contribute information through the marginal distribution of the observed feature \citep{dempster1977,mclachlanpeel2000,ahfock2023}. When the probability that a class label is missing depends on the observed feature, the missing-label pattern may itself be informative for the classifier. This possibility has been studied in Gaussian discrimination and more general likelihood-based and semi-supervised settings~\citep{ahfock2020,ahfock2023,wu2026, sportisse2023, lyu2024semi}. We therefore use a feature-dependent missing-at-random (MAR) model in which the missingness probability depends on classification uncertainty and shares parameters with the Weibull-mixture classifier. The present focus is on how this informative-missingness principle interacts with the non-Gaussian decision geometry of Weibull mixtures and, in turn, how the resulting information affects classification error through estimation of the Bayes decision boundaries.

The Weibull model gives a classification geometry that differs from the usual Gaussian case. Under a common component shape, the log-posterior odds are affine after a monotone power transformation, so a Bayes' rule has at most one positive decision boundary, which is unique when the rule is nonconstant. Under unequal component shapes, the log-posterior odds are nonlinear and can have two positive roots. One class may therefore occupy a bounded central interval and the other the two tails, or the roles may be reversed. This one- versus two-boundary structure is central to the classification-specific efficiency analysis developed below.

Our contributions are fourfold. First, we characterise the Bayes decision regions for common- and unequal-shape Weibull mixtures, including the two-boundary geometry that can arise under unequal shapes. Second, under a shared-parameter MAR label-missingness model, we derive the Fisher information for the classifier after adjustment for the nuisance parameters in the missingness mechanism, thereby separating information lost through unobserved labels from information gained through the missing-label indicators. Third, we derive a decision-boundary representation of the first-order expected error rate of the plug-in sample rule relative to the Bayes error. The expansion depends on estimation uncertainty only in parameter directions that perturb the decision boundaries, yields classification-specific asymptotic relative efficiency formulas for both the one- and two-boundary cases, and shows that a positive-definite increase in Fisher information is sufficient, but not necessary, for a smaller first-order expected error rate. Fourth, numerical simulations and a semi-synthetic study based on hard-drive failure data illustrate the finite-sample consequences for expected error rate and decision-boundary estimation. Relative to the Gaussian and more general informative-label-missingness settings cited above, the emphasis here is therefore on the way Weibull-specific decision geometry translates information in the missing-label pattern into classification efficiency.

\section{A two-component mixture of Weibull distributions for classification}\label{sec:weibull}

\subsection{Model and Bayes' rule}

Let $Y>0$ be a scalar feature and let $Z\in\{1,2\}$ denote its class. The class prior probabilities are
\[
\Prb(Z=i)=\pi_i,\qquad \pi_i>0,\qquad \pi_1+\pi_2=1.
\]
Conditional on $Z=i$, assume
\begin{equation*}
 f_i(y;k_i,\lambda_i)
 =\frac{k_i}{\lambda_i}\left(\frac{y}{\lambda_i}\right)^{k_i-1}
 \exp\left\{-\left(\frac{y}{\lambda_i}\right)^{k_i}\right\},
 \qquad y>0,
\end{equation*}
where $k_i>0$ and $\lambda_i>0$ are the shape and scale parameters. The corresponding distribution function is
\begin{equation*}
 F_i(y)=1-\exp\left\{-\left(\frac{y}{\lambda_i}\right)^{k_i}\right\}.
\end{equation*}
The marginal density of $Y$ is the two-component Weibull mixture
\begin{equation*}
 p_Y(y;\theta)=\pi_1 f_1(y)+\pi_2 f_2(y),
\end{equation*}
where $\theta$ denotes the Weibull-mixture parameter vector.

The posterior probability of membership in class 1 is
\begin{equation*}
 \tau_1(y;\theta)
 =\Prb(Z=1\mid Y=y)
 =\frac{\pi_1f_1(y)}{\pi_1f_1(y)+\pi_2f_2(y)},
 \qquad
 \logit\{\tau_1(y;\theta)\}=\delta(y;\theta),
\end{equation*}
where $\tau_2=1-\tau_1$ and
\begin{align}
 \delta(y;\theta)
 &=\log\frac{\pi_1f_1(y)}{\pi_2f_2(y)} \notag\\
 &=\log\frac{\pi_1}{\pi_2}+\log\frac{k_1}{k_2}
 -k_1\log\lambda_1+k_2\log\lambda_2
 +(k_1-k_2)\log y
 -\left(\frac{y}{\lambda_1}\right)^{k_1}
 +\left(\frac{y}{\lambda_2}\right)^{k_2}.
 \label{eq:delta-general}
\end{align}
Under zero-one loss, the Bayes' rule is
\begin{equation*}
 \mathcal C_\theta(y)=
 \begin{cases}
 1,&\delta(y;\theta)>0,\\
 2,&\delta(y;\theta)<0.
 \end{cases}
\end{equation*}
Ties occur only at roots of the log-posterior odds and hence on a set of probability zero in the cases considered below. A Bayes decision boundary is a positive root $y_b$ of the log-posterior odds function,
\begin{equation*}
 \delta(y_b;\theta)=0
 \quad\Longleftrightarrow\quad
 \pi_1 f_1(y_b)=\pi_2 f_2(y_b)
 \quad\Longleftrightarrow\quad
 \tau_1(y_b;\theta)=\tau_2(y_b;\theta)=\tfrac12.
\end{equation*}
For the asymptotic analysis in \Cref{sec:risk}, we restrict attention to simple roots, which locally separate regions assigned to different classes; a double root is instead a nonregular touching point. Throughout the paper, ``decision boundary'' is used in this classification sense.

For numerical and asymptotic calculations, we parameterise the original constrained parameters so that the transformed parameters range over the real line. Define
\[
 \rho=\log(\pi_1/\pi_2),\qquad
 \alpha_i=\log\lambda_i,\qquad
 \nu_i=\log k_i.
\]
Then $\logit(\pi_1)=\rho$, $\lambda_i=e^{\alpha_i}$, and $k_i=e^{\nu_i}$. For later use, set
\begin{equation*}
 a_i(y)=\log y-\alpha_i,
 \qquad
 r_i(y)=\exp\{k_i a_i(y)\}=\left(\frac{y}{\lambda_i}\right)^{k_i}.
\end{equation*}
The log-posterior odds function can be written as
\begin{equation*}
 \delta(y;\theta_U)
 =\rho+\nu_1-\nu_2+k_1a_1(y)-k_2a_2(y)-r_1(y)+r_2(y),
\end{equation*}
where
\[
 \theta_U=(\rho,\alpha_1,\nu_1,\alpha_2,\nu_2)^\top
\]
for the unequal-shape model.

\subsection{Common shape}
When $k_1=k_2=k$, let $\nu=\log k$. The constrained parameter vector is
\[
 \theta_C=(\rho,\alpha_1,\alpha_2,\nu)^\top.
\]
The log-posterior odds function simplifies to
\begin{align*}
 \delta_C(y;\theta_C)
 &=\rho+k\log(\lambda_2/\lambda_1)
 +\left(\lambda_2^{-k}-\lambda_1^{-k}\right)y^k \notag\\
 &=A+B y^k,
\end{align*}
where
\begin{equation*}
 A=\rho+k\log(\lambda_2/\lambda_1),
 \qquad
 B=\lambda_2^{-k}-\lambda_1^{-k}.
\end{equation*}
Consequently, the log-posterior odds function is affine in the transformed statistic $X=Y^k$, but is generally nonlinear in the original observation $Y$. Only when $k=1$ is it affine in $Y$ itself. Since $y^k$ is strictly increasing on $(0,\infty)$, a nonconstant Bayes' rule has a unique positive decision boundary if and only if $AB<0$, in which case
\begin{equation*}
 y_0=\left(-\frac{A}{B}\right)^{1/k}.
\end{equation*}

When $AB<0$, the Bayes error is available in closed form. If $B>0$,
\begin{equation*}
 \err_C(\theta_C)
 =\pi_1F_1(y_0)+\pi_2\{1-F_2(y_0)\},
\end{equation*}
and if $B<0$,
\begin{equation*}
 \err_C(\theta_C)
 =\pi_1\{1-F_1(y_0)\}+\pi_2F_2(y_0).
\end{equation*}
If no positive decision boundary exists, the Bayes' rule assigns every $y>0$ to the same class. Its error is $\pi_2$ when class 1 is always selected and $\pi_1$ when class 2 is always selected.

\subsection{Unequal shapes}
When $k_1\neq k_2$, the log-posterior odds function is nonlinear in both $y$ and $\log y$. Its derivative is
\begin{equation*}
 \delta_y(y;\theta_U)
 =\frac{k_1-k_2-k_1r_1(y)+k_2r_2(y)}{y}.
\end{equation*}
The next result describes the structure of the Bayes decision regions.

\begin{proposition}[Unequal-shape decision boundaries]\label{prop:unequal-boundary}
Suppose $k_1\neq k_2$.
\begin{enumerate}[label=(\alph*)]
\item The equation $\delta(y;\theta_U)=0$ has at most two positive roots.
\item If $k_1>k_2$, then $\delta(y)\to-\infty$ as $y\downarrow0$ and as $y\to\infty$. The log-posterior odds function has a unique global maximum. Hence it has either no positive root, one double root, or two simple roots $0<y_1<y_2$. In the two-root case, class 1 is selected on $(y_1,y_2)$ and class 2 in the two tails.
\item If $k_1<k_2$, then $\delta(y)\to+\infty$ as $y\downarrow0$ and as $y\to\infty$. The log-posterior odds function has a unique global minimum. Hence it has either no positive root, one double root, or two simple roots $0<y_1<y_2$. In the two-root case, class 2 is selected on $(y_1,y_2)$ and class 1 in the two tails.
\end{enumerate}
\end{proposition}

The proof is given in Appendix~\ref{app:roots}. This structure simplifies root finding: locate the unique stationary point of the log-posterior odds on the log scale, then bracket at most one root on each side. The result also supplies the decision-boundary geometry used in the ARE calculation. Figure~\ref{fig:decision-regions} illustrates these two forms of Bayes decision geometry: a single decision boundary under the common-shape model and two decision boundaries under the unequal-shape model.

\begin{figure}[!ht]
\centering

\begin{minipage}[t]{0.49\textwidth}
    \centering
    \includegraphics[width=\linewidth]
    {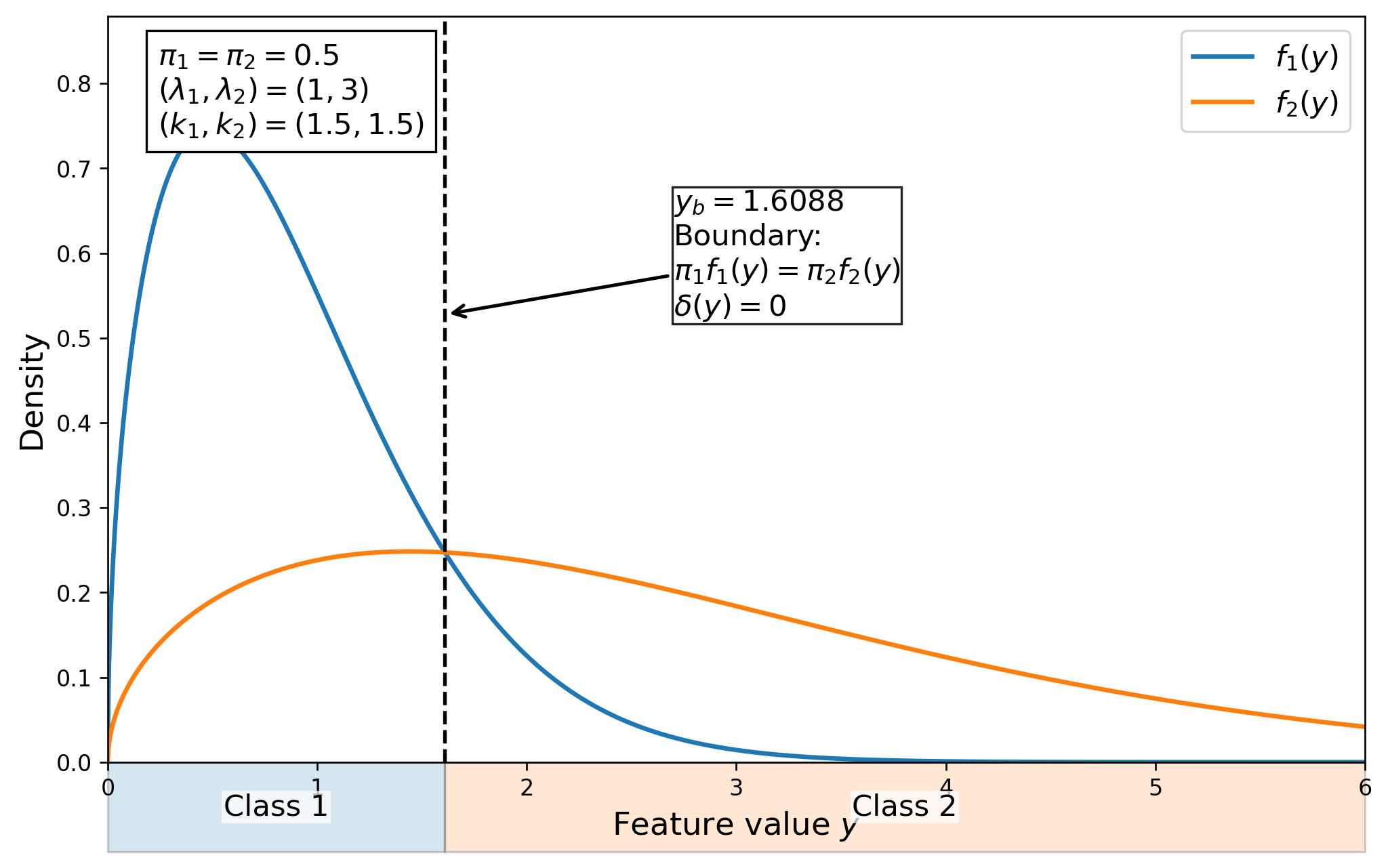}
    \par\smallskip
    \textbf{(a)}
\end{minipage}
\hfill
\begin{minipage}[t]{0.49\textwidth}
    \centering
    \includegraphics[width=\linewidth]
    {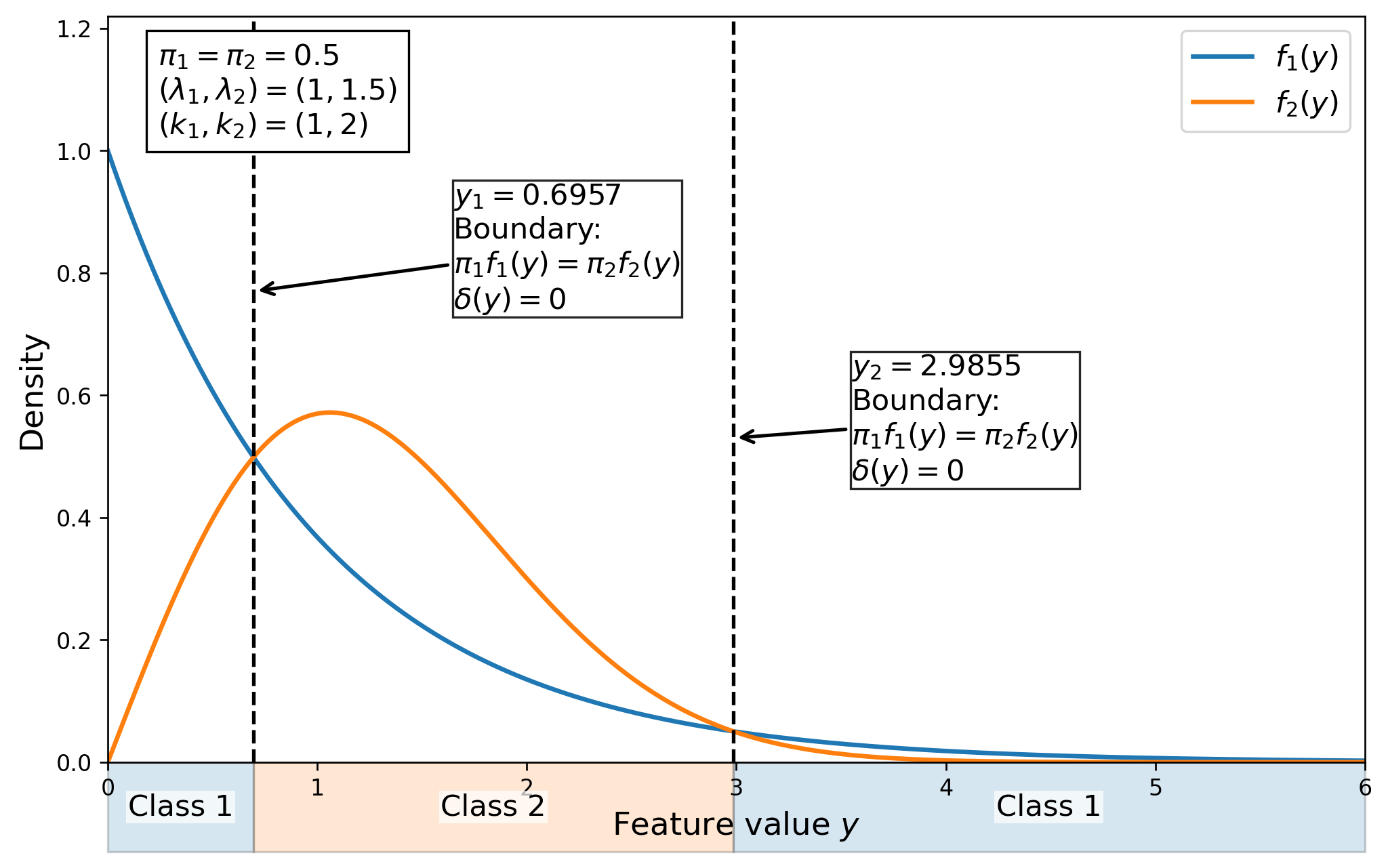}
    \par\smallskip
    \textbf{(b)}
\end{minipage}

\caption{
Bayes decision regions for representative two-component Weibull mixtures. Panel (a) shows the common-shape case with $\pi_1=\pi_2=0.5$, $(\lambda_1,\lambda_2)=(1,3)$, and $k_1=k_2=1.5$. The Bayes' rule has a single positive decision boundary at $y_b=1.6088$. Panel (b) shows the unequal-shape case with $\pi_1=\pi_2=0.5$, $(\lambda_1,\lambda_2)=(1,1.5)$, and $(k_1,k_2)=(1,2)$. The Bayes' rule has two positive decision boundaries, $y_1=0.6957$ and $y_2=2.9855$. The dashed vertical lines mark the roots of $\delta(y;\theta)=0$, equivalently the points at which $\pi_1 f_1(y)=\pi_2 f_2(y)$. The shaded bands indicate the corresponding Bayes decision regions.}
\label{fig:decision-regions}
\end{figure}

If $k_1>k_2$ and two roots exist, the Bayes error is
\begin{equation*}
 \err_U(\theta_U)
 =\pi_1\{F_1(y_1)+1-F_1(y_2)\}
 +\pi_2\{F_2(y_2)-F_2(y_1)\}.
\end{equation*}
If $k_1<k_2$ and two roots exist,
\begin{equation*}
 \err_U(\theta_U)
 =\pi_1\{F_1(y_2)-F_1(y_1)\}
 +\pi_2\{F_2(y_1)+1-F_2(y_2)\}.
\end{equation*}
When no positive roots exist, the Bayes' rule assigns all $y>0$ to the same class, and its error is the prior probability of the class that is never selected. The double-root case is excluded from the asymptotic boundary theory because $\delta_y=0$ at the touching point.

\subsection{Identifiability and label switching}
A finite mixture of univariate Weibull densities is identifiable up to a permutation of its components under the usual distinct-component conditions \citep{panteleeva2015}. In an entirely unlabelled sample, this leaves the familiar label-switching ambiguity. Partial class labels resolve that ambiguity provided both classes have positive probability of appearing among the labelled observations. 
\begin{assumption}[Identifiability conditions]\label{ass:identifiability}
The component parameter pairs $(k_1,\lambda_1)$ and $(k_2,\lambda_2)$ are distinct, $0<\pi_i<1$, and each class has positive probability of appearing among the labelled observations. In addition, $\delta(Y;\theta)^2$ is nondegenerate under the mixture, as required to identify both coefficients in the missingness model introduced below.
\end{assumption}
The distinct-component condition identifies the mixture up to permutation, while the labelled observations resolve the component-label permutation.

\section{Fisher information under the label-missingness model}\label{sec:information}

Here we consider a label-missingness mechanism for which the conditional probability of a missing label depends on the observed feature $Y$. Thus, the mechanism is MAR~\citep{rubin1976}. We retain the label-missingness mechanism in the statistical model because its linear predictor can depend on the same parameter vector that determines the log-posterior odds and hence Bayes' rule.

\subsection{Likelihoods for the partially classified sample}
\label{sec:likelihood}
Let $(Y_j,Z_j)$, $j=1,\ldots,n$, be independent observations from the two-component mixture of Weibull distributions. Let
\[
 M_j=
 \begin{cases}
 1,& Z_j\text{ is missing},\\
 0,& Z_j\text{ is observed}.
 \end{cases}
\]
The partially classified observation is $O_j=(Y_j,M_j,(1-M_j)Z_j)$. Write $z_{ij}=1$ if $Z_j=i$ and $0$ otherwise. The log likelihood formed from the features and their class labels, where available, while ignoring the label-missingness mechanism, is
\[ \ell_{\mathrm{ig}}(\theta)
 = \sum_{j=1}^n(1-M_j)\sum_{i=1}^2z_{ij}
 \log\{\pi_i f_i(Y_j)\} +\sum_{j=1}^n M_j\log p_Y(Y_j;\theta).
\]
If the sample were completely classified, the corresponding CC log likelihood would be
\begin{equation}
 \ell_{\mathrm{CC}}(\theta)
 =\sum_{j=1}^n\sum_{i=1}^2z_{ij}\log\{\pi_i f_i(Y_j)\}.
 \label{eq:cclik}
\end{equation}

We assume conditional independence of the missing-label indicator and the class label given the feature:
\begin{equation*}
 \Prb(M=1\mid Y=y,Z=z)=\Prb(M=1\mid Y=y)=q(y;\theta,\xi).
\end{equation*}
The missingness model is
\begin{equation}
 \logit\{q(y;\theta,\xi)\}
 =\eta(y;\theta,\xi)
 =\xi_0+\xi_1\delta(y;\theta)^2,
 \label{eq:qmodel}
\end{equation}
where $\xi=(\xi_0,\xi_1)^\top$. The case of interest is $\xi_1<0$, for which labels are most likely to be missing near a Bayes decision boundary. In estimation, however, $\xi_1$ is allowed to vary over $\mathbb R$; the inequality $\xi_1<0$ specifies the regime of primary interest rather than an optimisation constraint. When $\xi_1\neq0$, the missing-label probability varies with the observed feature and the mechanism is MAR. The case $\xi_1=0$ gives a constant missing-label probability and serves as an interior benchmark.

Although the mechanism is MAR in the feature-dependent case, \eqref{eq:qmodel} is a shared-parameter model: $q$ depends on the same $\theta$ as the class densities and Bayes' rule. Hence the parameter-distinctness condition ordinarily used for likelihood ignorability is not satisfied~\citep{rubin1976,lu2004missing}. The missing-label pattern can therefore provide information about the log-posterior odds function and the resulting Bayes' rule, so the label-missingness mechanism is retained in the likelihood.

The Bernoulli log likelihood for the missing-label indicators is
\begin{equation*}
 \ell_{\mathrm{miss}}(\theta,\xi)
 =\sum_{j=1}^n\left[
 M_j\log q(Y_j;\theta,\xi)
 +(1-M_j)\log\{1-q(Y_j;\theta,\xi)\}
 \right].
\end{equation*}
Combining these contributions gives the full log likelihood formed from the partially classified sample and its missing-label indicators
\begin{equation}
 \ell_{\mathrm{full}}(\theta,\xi)
 =\ell_{\mathrm{ig}}(\theta)+\ell_{\mathrm{miss}}(\theta,\xi).
 \label{eq:fulllik}
\end{equation}
We use the subscripts $\mathrm{CC}$, $\mathrm{ig}$, and $\mathrm{full}$ for the estimator based on the completely classified sample, the estimator based on the likelihood that ignores the label-missingness mechanism, and the estimator based on the full likelihood, respectively. Equivalently, a labelled observation contributes
\[
 \{1-q(y;\theta,\xi)\}\pi_z f_z(y),
\]
and an unlabelled observation contributes
\[
 q(y;\theta,\xi)p_Y(y;\theta).
\]

Because $\logit\{\tau_1(y)\}=\delta(y)$, entropy is an even, strictly decreasing function of $|\delta(y)|$. Hence $\delta(y)^2$ provides a smooth measure of classification certainty and attains its minimum at every Bayes decision boundary. Related uncertainty-dependent missing-label mechanisms have been used in likelihood-based semi-supervised classification~\citep{ahfock2023, lyu2024semi}. The expected missing-label fraction is
\begin{equation*}
 \gamma(\theta,\xi)=\E\{q(Y;\theta,\xi)\}.
\end{equation*}
All expectations below are with respect to the true mixture density unless stated otherwise.

\subsection{Log-posterior-odds derivatives and information in the completely classified sample}

Let
\[
 g(y;\theta)=\frac{\partial\delta(y;\theta)}{\partial\theta}.
\]
For the unequal-shape parametrisation $\theta_U=(\rho,\alpha_1,\nu_1,\alpha_2,\nu_2)^\top$,
\begin{equation*}
 g_U(y)=
 \begin{pmatrix}
 1\\[2pt]
 k_1(r_1-1)\\[2pt]
 1+k_1a_1(1-r_1)\\[2pt]
 k_2(1-r_2)\\[2pt]
 -1+k_2a_2(r_2-1)
 \end{pmatrix}.
\end{equation*}
For the common-shape parametrisation $\theta_C=(\rho,\alpha_1,\alpha_2,\nu)^\top$,
\begin{equation*}
 g_C(y)=
 \begin{pmatrix}
 1\\[2pt]
 k(r_1-1)\\[2pt]
 k(1-r_2)\\[2pt]
 k\{a_1(1-r_1)-a_2(1-r_2)\}
 \end{pmatrix}.
\end{equation*}
These derivatives also determine the information in both the class labels and the missing-label indicators.

Let $\gamma_E$ denote Euler's constant and define
\begin{equation*}
 C_W=\frac{\pi^2}{6}+(1-\gamma_E)^2.
\end{equation*}
For $Y\sim\mathrm{Weibull}(k,\lambda)$, the transformation
\[
 R=(Y/\lambda)^k
\]
has the standard exponential distribution. Under the log-scale/log-shape parametrisation $(\alpha,\nu)$, the class-conditional scores are
\begin{equation*}
 S_\alpha=k(R-1),
 \qquad
 S_\nu=1+\log R\,(1-R).
\end{equation*}
It follows that the per-observation Weibull information matrix for $(\alpha,\nu)$ is
\begin{equation*}
 W(k)=
 \begin{pmatrix}
 k^2 & k(\gamma_E-1)\\
 k(\gamma_E-1) & C_W
 \end{pmatrix}.
\end{equation*}

For unequal shapes, let
\[
 \theta_U=(\rho,\alpha_1,\nu_1,\alpha_2,\nu_2)^\top.
\]
The per-observation Fisher information in the completely classified sample is block diagonal:
\begin{equation*}
 \J_{\mathrm{CC},U}(\theta_U)
 =\diag\{\pi_1\pi_2,\,\pi_1W(k_1),\,\pi_2W(k_2)\}.
\end{equation*}

For a common shape, the per-observation information for $\theta_C=(\rho,\alpha_1,\alpha_2,\nu)^\top$ is
\begin{equation*}
 \J_{\mathrm{CC},C}=
 \begin{pmatrix}
 \pi_1\pi_2 & 0 & 0 & 0\\
 0 & \pi_1k^2 & 0 & \pi_1k(\gamma_E-1)\\
 0 & 0 & \pi_2k^2 & \pi_2k(\gamma_E-1)\\
 0 & \pi_1k(\gamma_E-1) & \pi_2k(\gamma_E-1) & C_W
 \end{pmatrix}.
\end{equation*}
The Fisher information in the completely classified sample is $\I_{\mathrm{CC}}=n\J_{\mathrm{CC}}$.

The quantities $\J_{\mathrm{lost}}$, $\J_{\mathrm{ig}}$, and $\J_{\mathrm{full}}$ considered below require one-dimensional integration because the posterior probabilities and missingness weights involve both components.

\subsection{Information loss and gain from the label-missingness mechanism}

For one observation, the log likelihood that ignores the label-missingness mechanism can be written as
\begin{equation}
 \ell_{\mathrm{ig},1}(\theta)
 =\log p(Y,Z;\theta)-M\log p(Z\mid Y;\theta).
 \label{eq:ig-decomp}
\end{equation}
Conditional on $Y=y$, the label is Bernoulli with success probability $\tau_1(y)$, where $\logit\{\tau_1(y)\}=\delta(y)$. Therefore, the conditional Fisher information about $\theta$ carried by $Z$, given $Y=y$, is
\begin{equation*}
 \J_{Z\mid Y}(y;\theta)
 =\tau_1(y)\tau_2(y)g(y;\theta)g(y;\theta)^\top.
\end{equation*}
Because $M\perp Z\mid Y$,
\begin{equation}
 \J_{\mathrm{lost}}(\theta,\xi)
 =\E\left[
 q(Y;\theta,\xi)\tau_1(Y;\theta)\tau_2(Y;\theta)
 g(Y;\theta)g(Y;\theta)^\top
 \right]
 \label{eq:Jlost}
\end{equation}
is the per-observation information lost through unobserved labels. Thus
\begin{equation}
 \J_{\mathrm{ig}}(\theta,\xi)
 =\J_{\mathrm{CC}}(\theta)-\J_{\mathrm{lost}}(\theta,\xi).
 \label{eq:Jig}
\end{equation}
Although $\ell_{\mathrm{ig}}(\theta)$ does not contain $\xi$, the matrix $\J_{\mathrm{ig}}$ depends on $\xi$ through the distribution of $M$. Under the shared-parameter model, $\ell_{\mathrm{ig}}$ is an estimating criterion rather than the full observed-data log likelihood. Nevertheless, its score satisfies an information equality at the true parameter, which justifies the inverse-information covariance used below.

\begin{lemma}[Information equality for the ignoring estimator]\label{lem:ig-information}
Let
\[
 S_Y(\theta)=\frac{\partial}{\partial\theta}\log p_Y(Y;\theta),
 \qquad
 S_{Z\mid Y}(\theta)=\frac{\partial}{\partial\theta}\log p(Z\mid Y;\theta),
\]
and let
\[
 \psi_{\mathrm{ig}}(\theta)
 =S_Y(\theta)+(1-M)S_{Z\mid Y}(\theta)
\]
be the score of $\ell_{\mathrm{ig},1}$. Under the conditional-independence assumption $M\perp Z\mid Y$ and the usual differentiability and moment conditions,
\begin{equation}
 -\E\left\{\frac{\partial\psi_{\mathrm{ig}}(\theta)}
 {\partial\theta^\top}\right\}
 =
 \E\{\psi_{\mathrm{ig}}(\theta)\psi_{\mathrm{ig}}(\theta)^\top\}
 =
 \J_{\mathrm{ig}}(\theta,\xi).
 \label{eq:ig-information-equality}
\end{equation}
If, in addition, $\widehat\theta_{\mathrm{ig}}\xrightarrow{p}\theta$, $\J_{\mathrm{ig}}$ is nonsingular, and the standard local regularity conditions for M-estimation hold, then
\begin{equation}
 \sqrt n(\widehat\theta_{\mathrm{ig}}-\theta)
 \xrightarrow{d}
 N\!\left(0,\J_{\mathrm{ig}}^{-1}\right).
 \label{eq:ig-asymptotic-normality}
\end{equation}
\end{lemma}
A proof is given in Appendix~\ref{app:information}.

Let
\begin{equation*}
 s(y;\theta)=\delta(y;\theta)^2,
 \qquad
 x(y;\theta)=\begin{pmatrix}1\\s(y;\theta)\end{pmatrix},
 \qquad
 w(y;\theta,\xi)=q(y)\{1-q(y)\}.
\end{equation*}
The gradient of the missingness linear predictor is
\begin{equation*}
 \frac{\partial\eta}{\partial\theta}
 =2\xi_1\delta(y;\theta)g(y;\theta)
 \equiv u(y;\theta,\xi),
 \qquad
 \frac{\partial\eta}{\partial\xi}=x(y;\theta).
\end{equation*}
The per-observation Fisher information from $M\mid Y$ for $(\theta^\top,\xi^\top)^\top$ has the following blocks, with their dependence on $(\theta,\xi)$ suppressed for brevity:
\begin{align*}
 B_{\theta\theta}
 &=\E\{w(Y)u(Y)u(Y)^\top\},\\
 B_{\theta\xi}
 &=\E\{w(Y)u(Y)x(Y)^\top\},\\
 B_{\xi\xi}
 &=\E\{w(Y)x(Y)x(Y)^\top\}.
\end{align*}
When $\xi$ is estimated jointly with $\theta$, the information about $\theta$ contributed by the missing-label indicators, after adjustment for $\xi$, is the Schur complement
\begin{equation}
 \J_{\mathrm{miss}}^{\mathrm{adj}}(\theta,\xi)
 =B_{\theta\theta}
 -B_{\theta\xi}B_{\xi\xi}^{-1}B_{\xi\theta}.
 \label{eq:Jmissadj}
\end{equation}
This positive-semidefinite matrix is the part of the missingness information about $\theta$ that remains after adjusting for the nuisance parameters $(\xi_0,\xi_1)$.

\begin{theorem}[Information under the full likelihood]\label{thm:information} 
Assume the model is identifiable and that the true parameter lies in an identifiable regular interior neighbourhood. Assume further that differentiation may be interchanged with integration, $B_{\xi\xi}$ is nonsingular, the relevant full-likelihood and completely classified likelihood maximisers are consistent for the true parameter, and the standard local regularity conditions for maximum-likelihood estimation hold. The per-observation Fisher information for $\theta$, after adjustment for the nuisance parameter $\xi$, is
\begin{equation}
 \J_{\mathrm{full}}(\theta,\xi)
 =\J_{\mathrm{CC}}(\theta)
 -\J_{\mathrm{lost}}(\theta,\xi)
 +\J_{\mathrm{miss}}^{\mathrm{adj}}(\theta,\xi).
 \label{eq:main-info-decomp}
\end{equation}
Consequently,
\begin{equation*}
 \J_{\mathrm{full}}(\theta,\xi)-\J_{\mathrm{CC}}(\theta)
 =\J_{\mathrm{miss}}^{\mathrm{adj}}(\theta,\xi)
 -\J_{\mathrm{lost}}(\theta,\xi).
\end{equation*}
If
\begin{equation}
 \J_{\mathrm{miss}}^{\mathrm{adj}}(\theta,\xi)
 -\J_{\mathrm{lost}}(\theta,\xi)>0,
 \label{eq:info-dominance}
\end{equation}
and both $\J_{\mathrm{full}}(\theta,\xi)$ and $\J_{\mathrm{CC}}(\theta)$ are positive definite, then the estimator based on the full likelihood has a smaller asymptotic covariance matrix than the maximum-likelihood estimator based on the completely classified sample.
\end{theorem}

A derivation is given in Appendix~\ref{app:information}. In subsequent displays, parameter arguments of the information matrices are suppressed when no ambiguity arises.

\begin{remark}[Role of the missingness parameters]
If $\xi$ were known, the missingness information would be $B_{\theta\theta}$; estimating $\xi$ subtracts the nuisance adjustment $B_{\theta\xi}B_{\xi\xi}^{-1}B_{\xi\theta}$. If $\xi_1=0$, then $\J_{\mathrm{miss}}^{\mathrm{adj}}=0$ and $\J_{\mathrm{full}}=\J_{\mathrm{CC}}-\J_{\mathrm{lost}}\leq\J_{\mathrm{CC}}$. Thus, a constant missing-label probability cannot offset the loss of class labels; any such offset requires feature-dependent MAR label missingness that is informative about $\theta$ through the shared-parameter structure.
\end{remark}

\begin{remark}[Augmented benchmark based on the completely classified sample]\label{rem:augmented-complete}
Consider the augmented experiment in which $(Y,Z,M)$ is observed. After adjustment for the nuisance parameter $\xi$, its information about $\theta$ is
\begin{equation*}
 \J_{\mathrm{CC}}^{\mathrm{aug}}
 =\J_{\mathrm{CC}}+\J_{\mathrm{miss}}^{\mathrm{adj}}.
\end{equation*}
Consequently,
\begin{equation*}
 \J_{\mathrm{full}}
 =\J_{\mathrm{CC}}^{\mathrm{aug}}-\J_{\mathrm{lost}}
 \leq \J_{\mathrm{CC}}^{\mathrm{aug}}.
\end{equation*}
Thus, the partially classified sample analysed with the full likelihood cannot dominate the augmented experiment that observes $(Y,Z,M)$.
\end{remark}

\begin{remark}
For any integrable function $h$,
\begin{equation*}
 \E\{h(Y)\}
 =\sum_{i=1}^2\pi_i\int_0^\infty
 h\left(\lambda_i r^{1/k_i}\right)e^{-r}\dd r.
\end{equation*}
Thus all required expectations, including $\gamma$, the Bayes error, and the ARE, can be evaluated by one-dimensional Gauss--Laguerre quadrature.
\end{remark}

\section{Expected error rate of the plug-in sample rule and asymptotic relative efficiency}\label{sec:risk}

\subsection{Expected error rate of the plug-in sample rule}
For an estimator $\widehat\theta_s$, the plug-in sample rule is $\mathcal C_{\widehat\theta_s}$. Its expected error rate, averaged over repeated training samples, is $\E\{R(\widehat\theta_s;\theta)\}$, where $\theta$ denotes the data-generating parameter. Because the Bayes error is fixed by $\theta$, estimators can be compared through
\[
 \E\{R(\widehat\theta_s;\theta)\}-\err(\theta),
 \qquad s\in\{\mathrm{CC},\mathrm{ig},\mathrm{full}\},
\]
which gives the usual sample-size interpretation of asymptotic relative efficiency~\citep{efron1975,ahfock2020}.

\subsection{Decision-boundary expansion}
For the local expansion, let $\widetilde\theta$ denote a generic parameter value used to form the classification rule, while $\theta$ denotes the data-generating parameter. The error rate of the corresponding rule is
\begin{equation*}
 R(\widetilde\theta;\theta)
 =\pi_1\Prb\{\delta(Y;\widetilde\theta)<0\mid Z=1\}
 +\pi_2\Prb\{\delta(Y;\widetilde\theta)>0\mid Z=2\},
\end{equation*}
where the probabilities are evaluated under the data-generating parameter $\theta$. At $\widetilde\theta=\theta$, this equals the Bayes error $\err(\theta)$. Throughout this section, $\delta$ denotes the log-posterior odds $\log\{\pi_1f_1/(\pi_2f_2)\}$ rather than an arbitrary rescaling of the decision score. Let
\[
 \mathcal{B}(\theta)=\{y>0:\delta(y;\theta)=0\}
\]
denote the Bayes decision boundaries.

\begin{assumption}[Stable simple boundaries]\label{ass:boundaries}
The Bayes' rule is nonconstant and $\mathcal{B}(\theta)=\{y_1,\ldots,y_K\}$ is finite. Each boundary lies in $(0,\infty)$ and is simple, $\delta_y(y_j;\theta)\neq0$. In a neighbourhood of $\theta$, $\delta(y;\widetilde\theta)$ is twice continuously differentiable in $(y,\widetilde\theta)$, the number of roots remains $K$, and the roots can be represented by differentiable functions $y_j(\widetilde\theta)$ that do not merge or escape to $0$ or $\infty$.
\end{assumption}
For the common-shape model, $K=1$ when $AB<0$; this structure is locally stable whenever $A$ and $B$ remain bounded away from zero. For the unequal-shape model, the two-boundary structure is locally stable when the unique extremum is strictly separated from zero, and both roots are simple. Tangent roots, root mergers, and constant Bayes' rules are nonregular cases and are excluded from the expansion below.

\begin{theorem}[Decision-boundary representation of the error-rate Hessian]\label{thm:risk-hessian}
Under \Cref{ass:boundaries}, with $p_Y$ continuous at the boundaries,
\begin{equation*}
 \left.\frac{\partial R(\widetilde\theta;\theta)}
 {\partial\widetilde\theta}\right|_{\widetilde\theta=\theta}=0,
\end{equation*}
and
\begin{equation}
 H_{R}(\theta)
 :=\left.\frac{\partial^2 R(\widetilde\theta;\theta)}
 {\partial\widetilde\theta\partial\widetilde\theta^\top}
 \right|_{\widetilde\theta=\theta}
 =\sum_{j=1}^K
 \frac{p_Y(y_j;\theta)}{2|\delta_y(y_j;\theta)|}
 g(y_j;\theta)g(y_j;\theta)^\top.
 \label{eq:risk-hessian}
\end{equation}
Let $\widehat\theta$ satisfy
\begin{equation}
 \sqrt n(\widehat\theta-\theta)\xrightarrow{d}N(0,V),
 \qquad
 n\E\{(\widehat\theta-\theta)(\widehat\theta-\theta)^\top\}\longrightarrow V,
 \label{eq:moment-convergence}
\end{equation}
and suppose that, for some $\epsilon>0$,
\begin{equation}
 \E\|\sqrt n(\widehat\theta-\theta)\|^{2+\epsilon}=O(1).
 \label{eq:higher-moment-condition}
\end{equation}
Then
\begin{equation}
 \E\{R(\widehat\theta;\theta)\}-\err(\theta)
 =\frac{1}{4n}\sum_{j=1}^K
 \frac{p_Y(y_j;\theta)}{|\delta_y(y_j;\theta)|}
 g(y_j;\theta)^\top Vg(y_j;\theta)
 +o(n^{-1}).
 \label{eq:excess-risk-generalV}
\end{equation}
For a correctly specified regular maximum-likelihood estimator satisfying the stated moment conditions,
$V=\J(\theta)^{-1}$, so
\begin{equation}
 \E\{R(\widehat\theta;\theta)\}-\err(\theta)
 =\frac{1}{4n}\sum_{j=1}^K
 \frac{p_Y(y_j;\theta)}{|\delta_y(y_j;\theta)|}
 g(y_j;\theta)^\top\J(\theta)^{-1}g(y_j;\theta)
 +o(n^{-1}).
 \label{eq:excess-risk}
\end{equation}
For the estimator that ignores the missingness mechanism, the same specialisation holds with $\J=\J_{\mathrm{ig}}$ by \Cref{lem:ig-information}.
\end{theorem}

The weight $p_Y(y_j)/|\delta_y(y_j)|$ measures the probability mass near boundary $y_j$ relative to the local steepness of the log-posterior odds. The quadratic form measures uncertainty in the estimated log-posterior odds at that boundary.

\subsection{Asymptotic relative efficiency}
For an estimator $s\in\{\mathrm{CC},\mathrm{ig},\mathrm{full}\}$, define the leading excess-error coefficient
\begin{equation}
 C_s(\theta)
 =\frac14\sum_{j=1}^K\omega_j
 g_j^\top V_sg_j,
 \qquad
 \omega_j=\frac{p_Y(y_j;\theta)}{|\delta_y(y_j;\theta)|},
 \qquad
 g_j=g(y_j;\theta),
 \label{eq:Cs}
\end{equation}
so that $\E\{R(\widehat\theta_s;\theta)\}-\err(\theta) =C_s(\theta)/n+o(n^{-1})$. Under the usual regularity conditions for the completely classified and full maximum-likelihood estimators, together with \Cref{lem:ig-information} for the ignoring estimator, the asymptotic covariance matrices are
\begin{equation*}
 V_{\mathrm{CC}}=\J_{\mathrm{CC}}^{-1},
 \qquad
 V_{\mathrm{ig}}=\J_{\mathrm{ig}}^{-1},
 \qquad
 V_{\mathrm{full}}=\J_{\mathrm{full}}^{-1}.
\end{equation*}
For $s\in\{\mathrm{ig},\mathrm{full}\}$, define the asymptotic relative efficiency with respect to the estimator based on the completely classified sample by
\begin{equation}
 \ARE_{s:\mathrm{CC}}(\theta)
 =\frac{C_{\mathrm{CC}}(\theta)}{C_s(\theta)}.
 \label{eq:ARE-def}
\end{equation}
Thus, $\ARE_{s:\mathrm{CC}}>1$ means that estimator $s$ has a smaller first-order expected error rate above the Bayes error than the estimator based on the completely classified sample. The main theoretical comparison below takes $s=\mathrm{full}$.

\begin{corollary}[Common-shape asymptotic relative efficiency]
\label{cor:ARE-common}
Assume the conditions of \Cref{thm:risk-hessian}. Under a common-shape two-component Weibull mixture, suppose $AB<0$, so that Bayes' rule has the unique simple decision boundary $y_0$. For the $\mathrm{full}$ estimator relative to the estimator based on the completely classified sample,
\begin{equation}
\ARE_{\mathrm{full}:\mathrm{CC}}^{\mathrm{C}}
=
\frac{
g_C(y_0)^\top
\J_{\mathrm{CC},C}^{-1}
g_C(y_0)
}{
g_C(y_0)^\top
\J_{\mathrm{full},C}^{-1}
g_C(y_0)
}.
\label{eq:ARE-common}
\end{equation}
\end{corollary}

\begin{proof}
For the common-shape model, $K=1$ in \eqref{eq:Cs}. Substituting $V_{\mathrm{CC}}=\J_{\mathrm{CC},C}^{-1}$ and $V_{\mathrm{full}}=\{\J_{\mathrm{full},C}\}^{-1}$ into \eqref{eq:ARE-def}, the common boundary factor $\omega_0/4$ cancels, yielding \eqref{eq:ARE-common}.
\end{proof}

The subscript $C$ denotes the common-shape parametrisation. The information matrices in \eqref{eq:ARE-common} must be obtained from the constrained common-shape likelihood; they are not obtained by setting $\nu_1=\nu_2$ after inverting the unequal-shape information matrix.

\begin{corollary}[Unequal-shape asymptotic relative efficiency]
\label{cor:ARE-unequal}
Assume the conditions of \Cref{thm:risk-hessian}. Under an unequal-shape two-component Weibull mixture, suppose the Bayes' rule has two simple decision boundaries $0<y_1<y_2$. For the $\mathrm{full}$ estimator relative to the estimator based on the completely classified sample,
\begin{equation}
\ARE_{\mathrm{full}:\mathrm{CC}}^{\mathrm{U}}
=
\frac{
\displaystyle
\sum_{j=1}^{2}
\omega_j\,
g_U(y_j)^\top
\J_{\mathrm{CC},U}^{-1}
g_U(y_j)
}{
\displaystyle
\sum_{j=1}^{2}
\omega_j\,
g_U(y_j)^\top
\J_{\mathrm{full},U}^{-1}
g_U(y_j)
}.
\label{eq:ARE-unequal}
\end{equation}
\end{corollary}

\begin{proof}
For the unequal-shape model, $K=2$ in \eqref{eq:Cs}. Substitution of the covariance matrices for the completely classified and full-likelihood estimators into \eqref{eq:ARE-def} gives \eqref{eq:ARE-unequal}. Unlike the single-boundary case, the two boundary weights generally do not cancel.
\end{proof}

The positive-definiteness condition in \eqref{eq:info-dominance} is sufficient for an ARE above one, but it is stronger than necessary: only covariance in the parameter directions that perturb the Bayes decision boundaries enters the first-order error-rate comparison.

\section{Numerical simulations}
\label{sec:numerical}

\subsection{Simulation setup}
\label{subsec:simulation-setup}

We conducted numerical simulations for six two-component Weibull mixture models with equal class probabilities, $\pi_1=\pi_2=0.5$, sample size $n=500$, and $B=500$ Monte Carlo replications. \Cref{tab:simulation-designs} summarises the data-generating models. Designs C1--C3 impose a common Weibull shape parameter and therefore have one positive Bayes decision boundary, whereas designs U1--U3 allow unequal shape parameters and have two.

\begin{table}[!ht]
\centering
\caption{Weibull-mixture models used in the numerical simulations.}
\label{tab:simulation-designs}
\begin{tabular}{@{}llcccc@{}}
\toprule
Design
& Shape model
& $(\lambda_1,\lambda_2)$
& $(k_1,k_2)$
& Bayes error
& Decision boundary or boundaries \\
\midrule
C1 & Common  & $(1,3)$   & $(1,1)$     & 0.3075
   & 1.6479 \\
C2 & Common  & $(1,3)$   & $(1.5,1.5)$ & 0.2274
   & 1.6088 \\
C3 & Common  & $(1,3)$   & $(3,3)$     & 0.0758
   & 1.5070 \\
U1 & Unequal & $(1,1.5)$ & $(0.7,1.5)$ & 0.3266
   & 0.4882,\ 3.4944 \\
U2 & Unequal & $(1,1.5)$ & $(1,2)$     & 0.3304
   & 0.6957,\ 2.9855 \\
U3 & Unequal & $(1,1.5)$ & $(1.5,3)$   & 0.3051
   & 0.9185,\ 2.5641 \\
\bottomrule
\end{tabular}

\medskip
\begin{minipage}{0.95\linewidth}
\footnotesize
\textit{Note.}
The common-shape models have one positive decision boundary, whereas the unequal-shape models have two positive simple decision boundaries.
\end{minipage}
\end{table}

In each simulation setting, class labels were removed according to
\begin{equation*}
\logit\{q(y;\theta,\xi)\}
=
\xi_0+\xi_1\delta(y;\theta)^2,
\end{equation*}
where $\delta(y;\theta)$ is the log-posterior odds defined in
\eqref{eq:delta-general}.
Because $\delta(y;\theta)=0$ at each Bayes decision boundary, $\delta(y;\theta)^2$ provides a smooth measure of classification certainty and attains its minimum at the boundaries. We considered
\[
\xi_1\in\{0,-0.5,-1,-2\}.
\]
For $\xi_1<0$, label missingness is MAR because the missing-label probability depends on the observed feature; more negative values correspond to stronger concentration of missing labels near the Bayes decision boundaries. The value $\xi_1=0$ provides the constant-probability benchmark. For each value of $\xi_1$, the intercept $\xi_0$ was calibrated to satisfy
\begin{equation*}
\E\{q(Y;\theta,\xi)\}
=
\gamma,
\qquad
\gamma\in\{0.10,0.30,0.50\}.
\end{equation*}
The resulting expected missing-label proportions were 10\%, 30\%, and 50\%. The negative values of $\xi_1$ therefore generate increasingly feature-dependent MAR label missingness. The simulation study is intended to assess finite-sample performance rather than to provide a numerical verification of the asymptotic ARE formulas. The six Weibull models, four values of $\xi_1$, and three missing-label proportions yielded 72 simulation settings.

We compared three estimators: $\mathrm{CC}$, obtained from the likelihood for the completely classified sample in \eqref{eq:cclik}; $\mathrm{ig}$, obtained from the likelihood in Section~\ref{sec:likelihood} that ignores the label-missingness mechanism; and $\mathrm{full}$, obtained from the full likelihood in \eqref{eq:fulllik}, which jointly estimates $\theta$ and $(\xi_0,\xi_1)$. For each fitted model, we evaluated the error rate of the plug-in rule directly under the data-generating distribution, avoiding Monte Carlo error from a finite test sample.

For method $s\in\{\mathrm{CC},\mathrm{ig},\mathrm{full}\}$, the mean excess error rate was defined as
\begin{equation*}
\overline{E}_s
=
\frac{1}{B}
\sum_{b=1}^{B}
\left[
R\left(
\widehat{\theta}_s^{(b)};\theta
\right)
-
\err(\theta)
\right].
\end{equation*}
The empirical relative efficiency of method $s$, relative to the estimator based on the completely classified sample, was calculated as
\begin{equation*}
\widehat{\RE}_{s:\mathrm{CC}}
=
\frac{\overline{E}_{\mathrm{CC}}}
{\overline{E}_s}.
\end{equation*}
A value greater than one indicates that method $s$ has a smaller mean excess error rate than the benchmark based on the completely classified sample.

\begin{remark}[Computation]
Each likelihood was maximised numerically from five starting values. Among the runs returning a finite objective value, the solution with the largest log likelihood was retained. A fit was classified as converged when the numerical optimiser returned a successful convergence status. Convergence rates ranged from 98\% to 100\% across all methods and simulation settings, with most settings attaining 100\% convergence. Nonconverged and failed fits were retained in the simulation records but excluded from the Monte Carlo summaries, which were calculated from the successfully converged replications.
\end{remark}

\subsection{Expected error-rate efficiency}
\label{subsec:classification-performance}

\Cref{fig:relative-efficiency} summarises the empirical relative efficiencies across the six designs, and \Cref{tab:re-common,tab:re-unequal} report the corresponding numerical values. When $\xi_1=0$, the $\mathrm{full}$ and $\mathrm{ig}$ estimators perform similarly. As the MAR dependence on the observed feature strengthens with decreasing $\xi_1$, the $\mathrm{full}$ estimator shows progressively larger gains over the $\mathrm{ig}$ estimator.

The gain from modelling the missing-label mechanism becomes more pronounced as the missing-label proportion increases. Under the common-shape designs, the $\mathrm{full}$ estimator attains empirical relative efficiencies above six in several settings with strong MAR dependence on the observed feature. The unequal-shape designs exhibit the same qualitative pattern, although the gains are generally more moderate. By contrast, the empirical relative efficiency of the $\mathrm{ig}$ estimator remains below one in nearly all settings and decreases as the missing-label proportion increases.

\begin{figure}[!ht]
\centering
\includegraphics[width=\textwidth]
{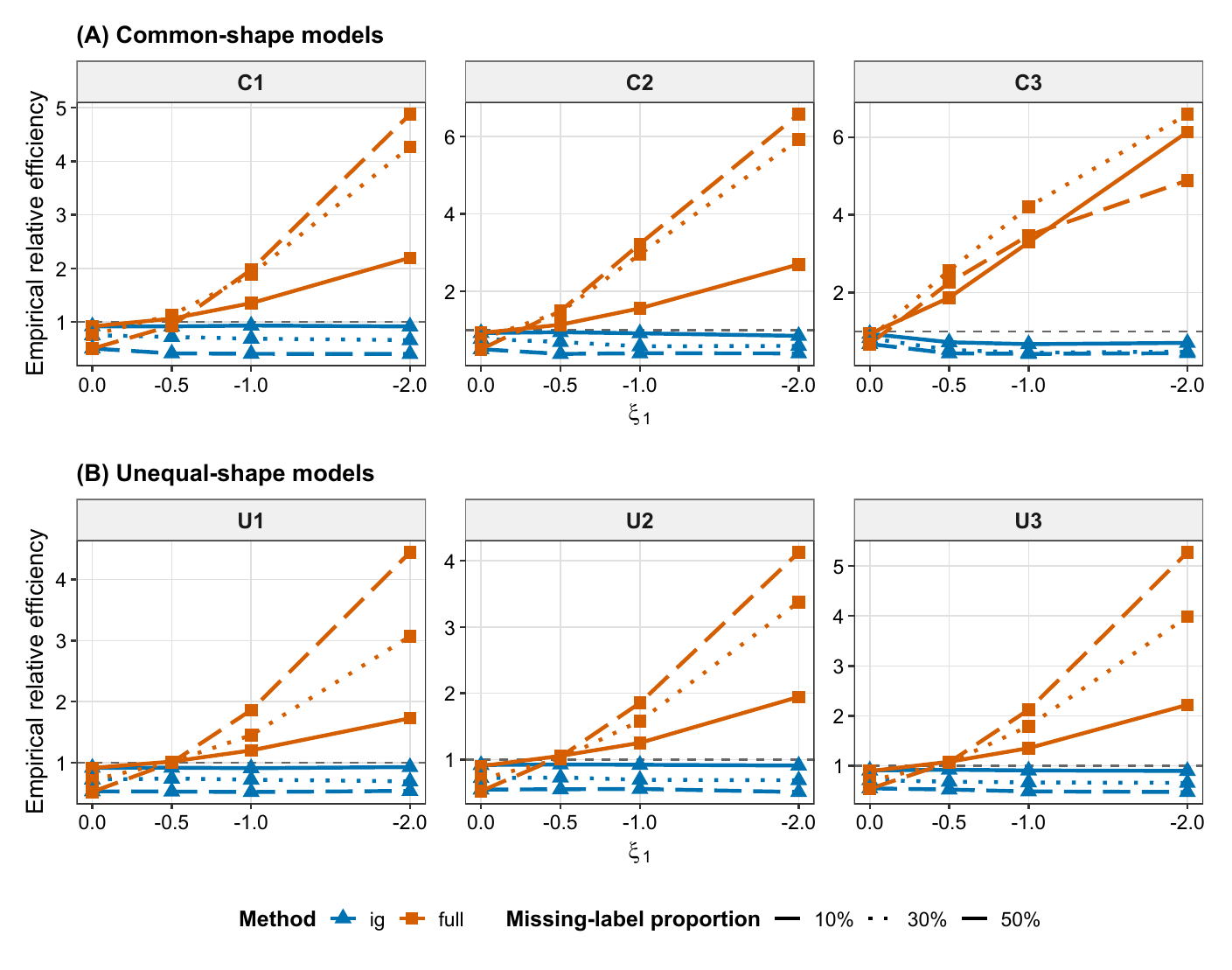}
\caption{
Empirical relative efficiency relative to the estimator based on the completely classified sample (CC). Panel (A) shows the common-shape designs and Panel (B) the unequal-shape designs. Blue and orange curves correspond to the $\mathrm{ig}$ and $\mathrm{full}$ estimators, respectively. Line types indicate the missing-label proportions (10\%, 30\%, and 50\%). The horizontal dashed line marks relative efficiency equal to one.
}
\label{fig:relative-efficiency}
\end{figure}

\begin{table}[!ht]
\centering
\caption{Empirical relative efficiencies under the common-shape Weibull models.}
\label{tab:re-common}
\begin{tabular}{@{}ccrrrrrr@{}}
\toprule
&
&
\multicolumn{2}{c}{10\% missing}
&
\multicolumn{2}{c}{30\% missing}
&
\multicolumn{2}{c}{50\% missing} \\
\cmidrule(lr){3-4}
\cmidrule(lr){5-6}
\cmidrule(lr){7-8}
Design
& $\xi_1$
& $\mathrm{ig}$ & $\mathrm{full}$
& $\mathrm{ig}$ & $\mathrm{full}$
& $\mathrm{ig}$ & $\mathrm{full}$ \\
\midrule
\multirow{4}{*}{C1}
& 0.0
& 0.918 & 0.917
& 0.752 & 0.757
& 0.503 & 0.500 \\
& -0.5
& 0.916 & 1.060
& 0.719 & 1.129
& 0.411 & 0.929 \\
& -1.0
& 0.931 & 1.350
& 0.686 & 1.883
& 0.404 & 1.973 \\
& -2.0
& 0.915 & 2.196
& 0.659 & 4.269
& 0.401 & 4.877 \\
\midrule
\multirow{4}{*}{C2}
& 0.0
& 0.925 & 0.925
& 0.774 & 0.775
& 0.506 & 0.505 \\
& -0.5
& 0.945 & 1.138
& 0.692 & 1.466
& 0.384 & 1.493 \\
& -1.0
& 0.916 & 1.562
& 0.581 & 2.959
& 0.401 & 3.228 \\
& -2.0
& 0.851 & 2.699
& 0.591 & 5.931
& 0.397 & 6.579 \\
\midrule
\multirow{4}{*}{C3}
& 0.0
& 0.941 & 0.941
& 0.831 & 0.831
& 0.680 & 0.680 \\
& -0.5
& 0.726 & 1.870
& 0.528 & 2.569
& 0.440 & 2.259 \\
& -1.0
& 0.676 & 3.292
& 0.457 & 4.210
& 0.425 & 3.473 \\
& -2.0
& 0.709 & 6.131
& 0.485 & 6.586
& 0.444 & 4.886 \\
\bottomrule
\end{tabular}
\end{table}

\begin{table}[!ht]
\centering
\caption{Empirical relative efficiencies under the unequal-shape Weibull models.}
\label{tab:re-unequal}
\begin{tabular}{@{}ccrrrrrr@{}}
\toprule
&
&
\multicolumn{2}{c}{10\% missing}
&
\multicolumn{2}{c}{30\% missing}
&
\multicolumn{2}{c}{50\% missing} \\
\cmidrule(lr){3-4}
\cmidrule(lr){5-6}
\cmidrule(lr){7-8}
Design
& $\xi_1$
& $\mathrm{ig}$ & $\mathrm{full}$
& $\mathrm{ig}$ & $\mathrm{full}$
& $\mathrm{ig}$ & $\mathrm{full}$ \\
\midrule
\multirow{4}{*}{U1}
& 0.0
& 0.918 & 0.916
& 0.727 & 0.718
& 0.533 & 0.524 \\
& -0.5
& 0.919 & 1.019
& 0.741 & 1.017
& 0.529 & 1.004 \\
& -1.0
& 0.912 & 1.202
& 0.723 & 1.451
& 0.521 & 1.862 \\
& -2.0
& 0.929 & 1.729
& 0.696 & 3.070
& 0.540 & 4.443 \\
\midrule
\multirow{4}{*}{U2}
& 0.0
& 0.917 & 0.901
& 0.720 & 0.706
& 0.547 & 0.524 \\
& -0.5
& 0.925 & 1.054
& 0.729 & 1.013
& 0.551 & 1.041 \\
& -1.0
& 0.920 & 1.251
& 0.694 & 1.582
& 0.555 & 1.853 \\
& -2.0
& 0.909 & 1.945
& 0.687 & 3.378
& 0.509 & 4.127 \\
\midrule
\multirow{4}{*}{U3}
& 0.0
& 0.907 & 0.894
& 0.711 & 0.703
& 0.549 & 0.540 \\
& -0.5
& 0.920 & 1.081
& 0.682 & 1.060
& 0.524 & 1.069 \\
& -1.0
& 0.907 & 1.353
& 0.667 & 1.792
& 0.487 & 2.121 \\
& -2.0
& 0.896 & 2.220
& 0.657 & 3.987
& 0.475 & 5.273 \\
\bottomrule
\end{tabular}
\end{table}

\subsection{Parameter estimation}
\label{subsec:parameter-estimation}

\Cref{fig:parameter-rmse} provides complementary evidence from parameter estimation. It reports empirical root mean squared errors for the mixing proportion and the Weibull scale and shape parameters. For visual clarity, the RMSEs are averaged over designs C1--C3 within the common-shape family and over designs U1--U3 within the unequal-shape family. At $\xi_1=0$, the $\mathrm{full}$ and $\mathrm{ig}$ estimators have broadly similar accuracy. As the MAR dependence on classification uncertainty strengthens, the $\mathrm{full}$ estimator generally has smaller RMSE, particularly at the larger missing-label proportions.

These reductions in parameter-estimation RMSE occur in both model families. The $\mathrm{CC}$ estimator provides a baseline reference, while the $\mathrm{full}$ estimator often approaches or improves upon the $\mathrm{CC}$ RMSE when MAR label missingness is concentrated near the decision boundary.

\begin{figure}[!ht]
\centering
\includegraphics[width=\textwidth]
{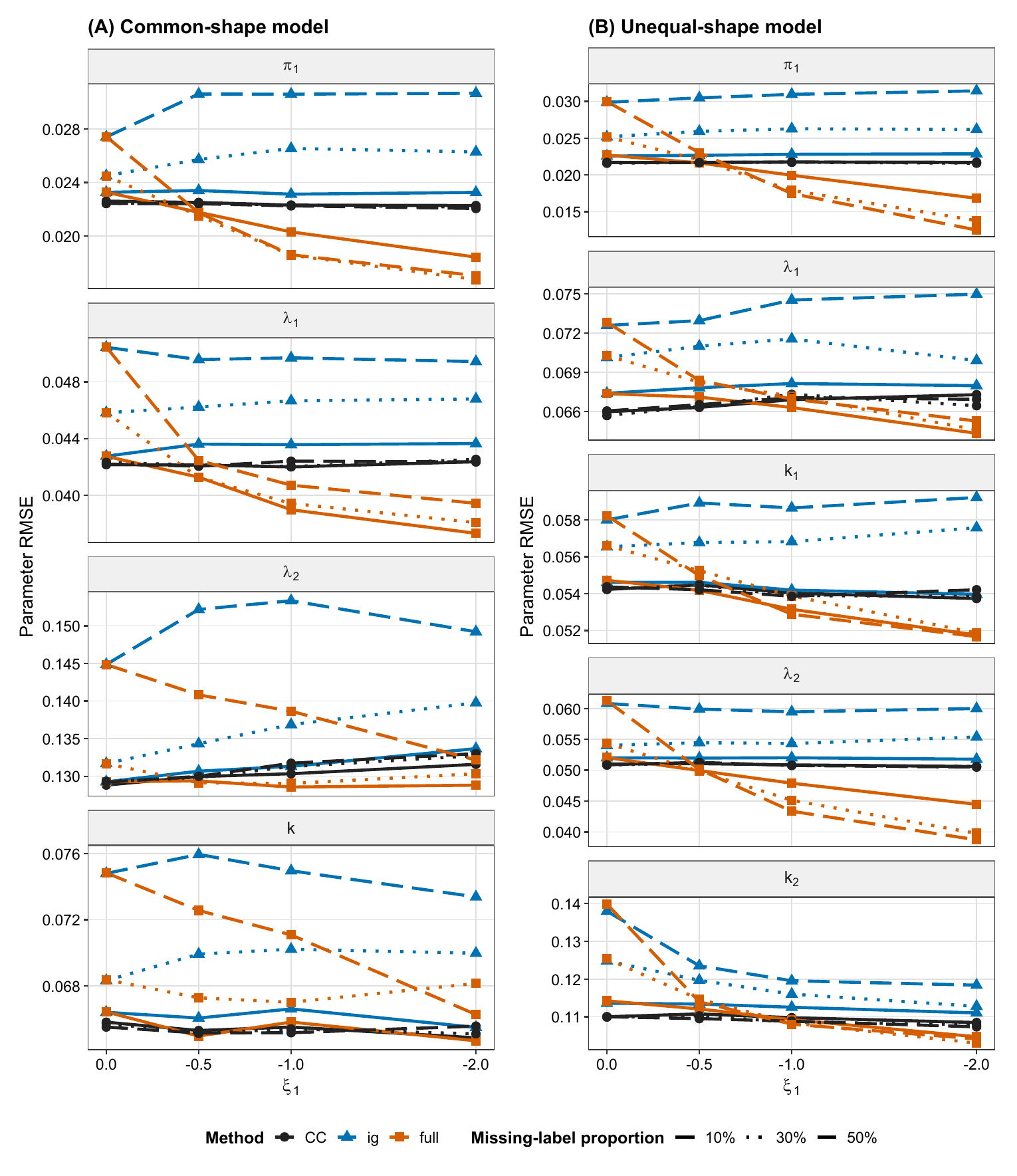}
\caption{
RMSEs of the estimated model parameters.
Panel (A) summarises the common-shape models, displaying the mixing proportion,
two scale parameters, and the common shape parameter.
Panel (B) summarises the unequal-shape models, displaying the mixing proportion,
two scale parameters, and two component-specific shape parameters.
RMSEs are averaged over the three designs within each shape-model family.
Black, blue, and orange curves correspond to the $\mathrm{CC}$,
$\mathrm{ig}$, and $\mathrm{full}$ estimators, respectively.
Line types indicate the missing-label proportions.
}
\label{fig:parameter-rmse}
\end{figure}

\subsection{Decision-boundary estimation}
\label{subsec:boundary-estimation}

The primary inferential target is the plug-in classification rule rather than the model parameters alone. We therefore evaluated the RMSE of the estimated Bayes decision boundaries. \Cref{fig:boundary-rmse} indicates that the $\mathrm{full}$ estimator generally yields smaller boundary RMSE than the $\mathrm{ig}$ estimator. The difference increases as the MAR dependence on classification uncertainty strengthens and as the missing-label proportion rises.

The common-shape designs contain a single positive decision boundary, whereas the unequal-shape designs contain two. In the unequal-shape case, the $\mathrm{full}$ estimator reduces RMSE for both boundaries, with particularly clear gains for the second boundary under strong MAR label missingness. These results indicate that modelling the missing-label mechanism improves estimation of the classification rule, not merely estimation of the mixture parameters.

\begin{figure}[!ht]
\centering
\includegraphics[width=\textwidth]
{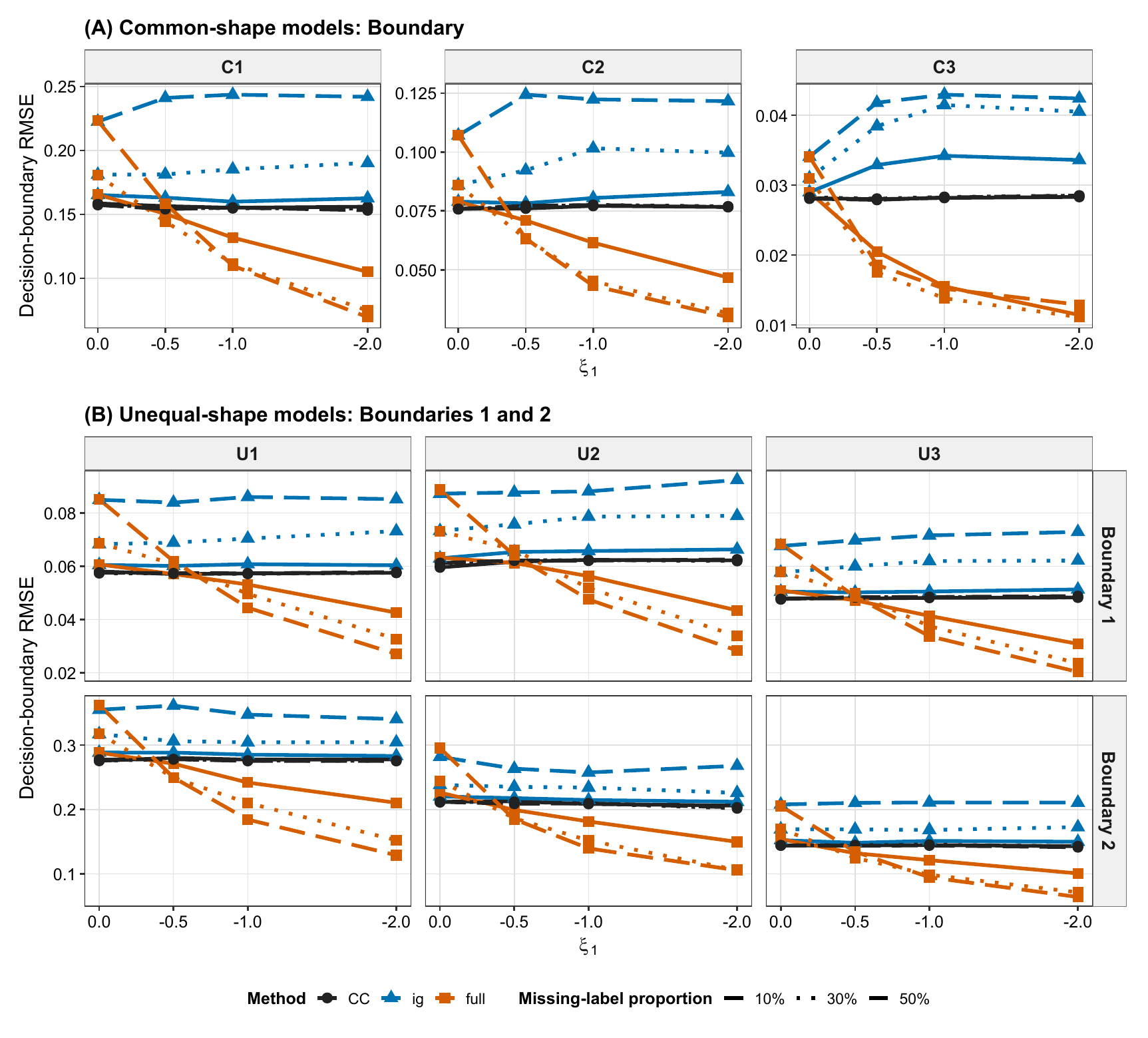}
\caption{
RMSEs of the estimated Bayes decision boundaries.
Panel (A) shows the single decision boundary for the common-shape models,
whereas Panel (B) shows the two decision boundaries for the unequal-shape models.
Black, blue, and orange curves correspond to the $\mathrm{CC}$,
$\mathrm{ig}$, and $\mathrm{full}$ estimators, respectively.
Line types indicate the missing-label proportions.
}
\label{fig:boundary-rmse}
\end{figure}

\section{A semi-synthetic study using hard-drive failure data}
\label{sec:hard-drive-classification}

We analysed Backblaze failure records for two frequently observed hard-drive product models~\cite{arslan2020distribution}. For drive $j$, the observed lifetime $Y_j$ was the positive-valued feature and $Z_j\in\{1,2\}$ denoted drive type, with \texttt{ST4000DM000} coded as class~1 and \texttt{ST8000DM002} as class~2. The objective was therefore to classify drive type among failed drives, not to predict future drive failure. After excluding the separate committee-training subset described in Appendix~\ref{app:hard-drive-experts} and two observations with zero recorded lifetime, the analysis sample contained $n=194$ drives: 98 from class~1 and 96 from class~2.

To emulate an expert-labelling setting in which ambiguous observations are more likely to remain unlabelled, we constructed three nested label-missingness regimes retaining approximately 60\%, 70\%, and 80\% of the class labels. These regimes represent stylised junior, intermediate, and senior levels of expert assessment. Labels were preferentially retained for observations that were easier to classify, whereas observations nearer the decision boundary were more likely to remain unlabelled. Details are given in Appendix~\ref{app:hard-drive-experts}.

Using the completely classified sample as the benchmark, we compared the common- and unequal-shape Weibull mixtures. BIC favoured the common-shape model ($2845.142$ versus $2847.951$), and the likelihood-ratio test did not reject the common-shape restriction ($\chi^2_1=2.459$, $p=0.117$). We therefore used the common-shape model in this study. At each expert level, we compared the completely classified estimator ($\mathrm{CC}$), the estimator that ignored the missingness mechanism ($\mathrm{ig}$), and the full-likelihood estimator ($\mathrm{full}$). The latter used the MAR working model in \eqref{eq:qmodel}, with missingness depending on the observed lifetime through the squared log-posterior odds.

Labels were more likely to be absent near the estimated Bayes decision boundary. The $\mathrm{CC}$ estimate was $\widehat y_{0,\mathrm{CC}}=287.94$, whereas the $\mathrm{ig}$ estimates were substantially lower. Incorporating the feature-dependent missing-label pattern moved the $\mathrm{full}$ estimates closer to the $\mathrm{CC}$ benchmark.

Predictive misclassification error was estimated by repeated stratified five-fold cross-validation. The same folds were used for all three expert regimes, so the cross-validated $\mathrm{CC}$ result was identical across regimes. The $\mathrm{full}$ estimator had the lowest mean estimated error rate at every expert level. Relative to $\mathrm{ig}$, the mean error rate decreased by $0.076$, $0.075$, and $0.053$ for the junior, intermediate, and senior experts, respectively; the corresponding decreases relative to $\mathrm{CC}$ were $0.036$, $0.036$, and $0.030$. The estimated decision boundaries and cross-validated error rates are summarised in Table~\ref{tab:hard-drive-results}.

\begin{table}[!ht]
\centering
\begin{threeparttable}
\caption{Estimated decision boundaries and mean estimated error rates.
Error rates were evaluated using five repetitions of five-fold cross-validation.}
\label{tab:hard-drive-results}
\small
\begin{tabular*}{\linewidth}{@{\extracolsep{\fill}}lrrrrrr@{}}
\toprule
& \multicolumn{3}{c}{Decision boundary}
& \multicolumn{3}{c}{Estimated error rate} \\
\cmidrule(lr){2-4}
\cmidrule(lr){5-7}
Expert
& $\mathrm{CC}$
& $\mathrm{ig}$
& $\mathrm{full}$
& $\mathrm{CC}$
& $\mathrm{ig}$
& $\mathrm{full}$ \\
\midrule
Junior
& 287.94 & 158.91 & 328.71
& 0.427 & 0.467 & \textbf{0.391} \\
Intermediate
& 287.94 & 150.83 & 343.38
& 0.427 & 0.466 & \textbf{0.391} \\
Senior
& 287.94 & 180.19 & 322.93
& 0.427 & 0.450 & \textbf{0.397} \\
\bottomrule
\end{tabular*}
\begin{tablenotes}[flushleft]
\footnotesize
\item Decision boundaries were estimated by fitting each method to the full analysis sample. Error rates are one minus the corresponding classification accuracies and are averaged over five repetitions of stratified five-fold cross-validation, with all models refitted in the corresponding training folds.
\end{tablenotes}
\end{threeparttable}
\end{table}

Therefore, modelling the feature-dependent missing-label mechanism produced decision-boundary estimates closer to the completely classified benchmark and reduced the cross-validated error rate in this example.

\section{Discussion}\label{sec:discussion}

We have studied semi-supervised classification from a two-component Weibull mixture under a shared-parameter MAR label-missingness model. The Weibull setting leads to one positive decision boundary under a nonconstant common-shape Bayes' rule and can lead to two under unequal shapes. After adjustment for nuisance parameters in the missingness model, the Fisher information and the expected error rate of the plug-in sample rule can both be expressed in forms that make the missing-label pattern explicit. The decision-boundary expansion shows why improved estimation of all model parameters is not required for a classification gain: the first-order error rate depends on uncertainty only in directions that perturb the decision boundaries. The simulations and the semi-synthetic analysis based on hard-drive failure data support this distinction. Extensions to nonregular boundary configurations, censored Weibull data, more flexible missingness models, and multiclass settings remain useful directions for further work \citep{ahfock2023,wu2026}.

\appendix

\section{Decision-boundary proof}\label{app:roots}
\begin{proof}[Proof of {\color{blue}Proposition 1}]
Set $t=\log y$ and $\phi(t)=\delta(e^t;\theta_U)$. Then
\begin{align*}
 \phi'(t)
 &=k_1-k_2-k_1\exp\{k_1(t-\alpha_1)\}
 +k_2\exp\{k_2(t-\alpha_2)\},\\
 \phi''(t)
 &=-k_1^2\exp\{k_1(t-\alpha_1)\}
 +k_2^2\exp\{k_2(t-\alpha_2)\}.
\end{align*}
The equation $\phi''(t)=0$ has exactly one solution because the ratio of its two exponential terms is strictly monotone when $k_1\neq k_2$.

If $k_1>k_2$, then $\phi''(t)>0$ for sufficiently negative $t$ and $\phi''(t)<0$ for sufficiently positive $t$. Hence $\phi'$ first increases and then decreases. Moreover,
\[
 \lim_{t\to-\infty}\phi'(t)=k_1-k_2>0,
 \qquad
 \lim_{t\to\infty}\phi'(t)=-\infty.
\]
Therefore $\phi'$ has exactly one zero and $\phi$ has a unique maximum. From \eqref{eq:delta-general},
\[
 \lim_{t\to-\infty}\phi(t)=-\infty,
 \qquad
 \lim_{t\to\infty}\phi(t)=-\infty.
\]
Thus $\phi$ has zero, one tangential zero at its unique maximum, or two simple zeros, and it is positive only between two simple zeros. It remains to verify that the tangential zero, when present, is exactly double. Let $t_\star$ denote the unique stationary point and write
\[
 r_i^\star=\exp\{k_i(t_\star-\alpha_i)\}>0.
\]
If both $\phi'(t_\star)=0$ and $\phi''(t_\star)=0$, then
\[
 k_1-k_2-k_1r_1^\star+k_2r_2^\star=0,
 \qquad
 -k_1^2r_1^\star+k_2^2r_2^\star=0.
\]
The second equation gives $r_2^\star=(k_1^2/k_2^2)r_1^\star$, and substitution into the first gives
\[
 (k_1-k_2)
 \left(1+\frac{k_1}{k_2}r_1^\star\right)=0,
\]
which is impossible because $k_1\neq k_2$ and $r_1^\star>0$. Hence $\phi''(t_\star)\neq0$. Therefore, if $\phi(t_\star)=0$, then $t_\star$ is a zero of multiplicity exactly two. Since the map $y=e^t$ is a smooth one-to-one reparameterization with nonzero derivative, the corresponding positive root of $\delta(y)$ also has multiplicity exactly two.

If $k_1<k_2$, the signs reverse: $\phi'$ first decreases and then increases,
\[
 \lim_{t\to-\infty}\phi'(t)=k_1-k_2<0,
 \qquad
 \lim_{t\to\infty}\phi'(t)=+\infty,
\]
so $\phi$ has a unique minimum. Also $\phi(t)\to+\infty$ in both tails. The same argument shows that a tangential zero at the minimum has multiplicity exactly two. The stated root count and class regions follow.
\end{proof}

\section{Derivation of the information identity}
\label{app:information}

\begin{proof}[Proof of \Cref{lem:ig-information}]
From \eqref{eq:ig-decomp},
\[
 \ell_{\mathrm{ig},1}(\theta)
 =
 \log p_Y(Y;\theta)
 +(1-M)\log p(Z\mid Y;\theta),
\]
so
\[
 \psi_{\mathrm{ig}}(\theta)
 =
 S_Y(\theta)+(1-M)S_{Z\mid Y}(\theta).
\]
Conditional score identities give
\[
 \E\{S_{Z\mid Y}(\theta)\mid Y\}=0,
 \qquad
 -\E\left\{
 \frac{\partial S_{Z\mid Y}(\theta)}
 {\partial\theta^\top}
 \,\middle|\,Y
 \right\}
 =
 \J_{Z\mid Y}(Y;\theta).
\]
Because $M\perp Z\mid Y$,
\[
 \E\{(1-M)S_{Z\mid Y}(\theta)\mid Y\}
 =
 \{1-q(Y)\}\E\{S_{Z\mid Y}(\theta)\mid Y\}=0.
\]
Hence the cross terms between $S_Y$ and $(1-M)S_{Z\mid Y}$ vanish. Since $(1-M)^2=1-M$,
\begin{align*}
 \E\{\psi_{\mathrm{ig}}\psi_{\mathrm{ig}}^\top\}
 &=
 \J_Y
 +\E\!\left[\{1-q(Y)\}\J_{Z\mid Y}(Y;\theta)\right],\\
 -\E\left\{\frac{\partial\psi_{\mathrm{ig}}}{\partial\theta^\top}\right\}
 &=
 \J_Y
 +\E\!\left[\{1-q(Y)\}\J_{Z\mid Y}(Y;\theta)\right],
\end{align*}
where $\J_Y$ is the Fisher information in the marginal feature density $p_Y$. The complete-data score decomposition
\[
 \log p(Y,Z;\theta)
 =
 \log p_Y(Y;\theta)+\log p(Z\mid Y;\theta)
\]
gives
\[
 \J_{\mathrm{CC}}
 =
 \J_Y+\E\{\J_{Z\mid Y}(Y;\theta)\}.
\]
Therefore,
\[
 \J_Y
 +\E\!\left[\{1-q(Y)\}\J_{Z\mid Y}(Y;\theta)\right]
 =
 \J_{\mathrm{CC}}
 -\E\{q(Y)\J_{Z\mid Y}(Y;\theta)\}
 =
 \J_{\mathrm{ig}},
\]
by \eqref{eq:Jlost}--\eqref{eq:Jig}. Thus the sensitivity and variability matrices of the ignoring score coincide, proving \eqref{eq:ig-information-equality}. Under the additional consistency, nonsingularity, and standard local M-estimation regularity conditions stated in \Cref{lem:ig-information}, the usual asymptotic linearization applies. Its sandwich covariance reduces to
\[
 \J_{\mathrm{ig}}^{-1}
 \J_{\mathrm{ig}}
 \J_{\mathrm{ig}}^{-1}
 =
 \J_{\mathrm{ig}}^{-1},
\]
which proves \eqref{eq:ig-asymptotic-normality}.
\end{proof}

\begin{proof}[Derivation for \Cref{thm:information}]
Since
\[
\ell_{\mathrm{full}}(\theta,\xi)
=
\ell_{\mathrm{ig}}(\theta)
+
\ell_{\mathrm{miss}}(\theta,\xi),
\]
linearity of the expected negative Hessian implies that the information
contributions from the ignored likelihood and the Bernoulli missingness
likelihood add, while the blocks involving $\xi$ arise only from the
missingness likelihood. By \Cref{lem:ig-information} and
\eqref{eq:Jig}, the $\theta\theta$ block contributed by
$\ell_{\mathrm{ig}}$ is $\J_{\mathrm{ig}}$. Hence the joint information
matrix for $(\theta^\top,\xi^\top)^\top$ is
\[
\begin{pmatrix}
\J_{\mathrm{ig}}+B_{\theta\theta} & B_{\theta\xi}\\
B_{\xi\theta} & B_{\xi\xi}
\end{pmatrix}.
\]
Profiling out the nuisance parameter $\xi$ gives the Schur complement
of $B_{\xi\xi}$:
\[
\J_{\mathrm{ig}}
+
B_{\theta\theta}
-
B_{\theta\xi}B_{\xi\xi}^{-1}B_{\xi\theta}.
\]
Substituting \eqref{eq:Jig} and \eqref{eq:Jmissadj} yields
\[
\J_{\mathrm{full}}
=
\J_{\mathrm{CC}}
-
\J_{\mathrm{lost}}
+
\J_{\mathrm{miss}}^{\mathrm{adj}},
\]
which proves \eqref{eq:main-info-decomp}. Moreover,
\[
\J_{\mathrm{full}}-\J_{\mathrm{CC}}
=
\J_{\mathrm{miss}}^{\mathrm{adj}}-\J_{\mathrm{lost}}.
\]
Thus, if \eqref{eq:info-dominance} holds and both
$\J_{\mathrm{full}}$ and $\J_{\mathrm{CC}}$ are positive definite, then
\[
\J_{\mathrm{full}}>\J_{\mathrm{CC}}.
\]
Inversion reverses the Loewner order, giving
\[
\J_{\mathrm{full}}^{-1}
<
\J_{\mathrm{CC}}^{-1}.
\]
\end{proof}

\section{Proof of the expected error-rate expansion}
\label{app:risk}

\begin{proof}[Proof of \Cref{thm:risk-hessian}]
Write
\begin{equation*}
 h(y)=\pi_1f_1(y)-\pi_2f_2(y).
\end{equation*}
Since
\begin{equation*}
 \tau_1(y;\theta)
 =\frac{e^{\delta(y;\theta)}}{1+e^{\delta(y;\theta)}},
\end{equation*}
we have
\begin{equation*}
 \begin{aligned}
 h(y)
 &=p_Y(y;\theta)\{2\tau_1(y;\theta)-1\}\\
 &=p_Y(y;\theta)
 \tanh\left\{\frac{\delta(y;\theta)}{2}\right\}.
 \end{aligned}
\end{equation*}

Let $y_j$ be a Bayes decision boundary. Then
$\delta(y_j;\theta)=0$, and hence $h(y_j)=0$. Because
$p_Y(\,\cdot\,;\theta)$ is continuous at $y_j$,
$\delta(\,\cdot\,;\theta)$ is differentiable at $y_j$, and
\begin{equation*}
 \tanh(u/2)=u/2+o(u)
 \qquad\text{as }u\to0,
\end{equation*}
it follows that
\begin{equation}
 h(y)
 =\frac{p_Y(y_j;\theta)}{2}
 \delta_y(y_j;\theta)(y-y_j)
 +o(|y-y_j|)
 \qquad\text{as }y\to y_j.
 \label{eq:h-local-expansion}
\end{equation}
In particular, $h$ is differentiable at $y_j$, with
\begin{equation*}
 h'(y_j)
 =\frac{p_Y(y_j;\theta)}{2}
 \delta_y(y_j;\theta).
\end{equation*}

By \Cref{ass:boundaries} and the implicit function theorem, for
$\widetilde\theta$ in a neighbourhood of $\theta$, the boundary
$y_j$ extends to a differentiable root $y_j(\widetilde\theta)$
satisfying
\begin{equation*}
 \delta\{y_j(\widetilde\theta);\widetilde\theta\}=0.
\end{equation*}
A first-order expansion about $(y_j,\theta)$ gives
\begin{equation}
 y_j(\widetilde\theta)-y_j
 =-\frac{g(y_j;\theta)^\top(\widetilde\theta-\theta)}
 {\delta_y(y_j;\theta)}
 +o(\|\widetilde\theta-\theta\|).
 \label{eq:root-gradient}
\end{equation}

The excess error rate relative to the Bayes error admits the
representation
\begin{equation*}
 R(\widetilde\theta;\theta)-R(\theta;\theta)
 =\int_0^\infty |h(y)|
 \mathbf 1\left\{
 \operatorname{sign}\delta(y;\widetilde\theta)
 \neq\operatorname{sign}\delta(y;\theta)
 \right\}\dd y.
\end{equation*}
By the stability and simplicity of the Bayes decision boundaries,
for $\widetilde\theta$ sufficiently close to $\theta$, disagreement
between the two rules occurs only on the disjoint intervals having
endpoints $y_j$ and $y_j(\widetilde\theta)$.

Using \eqref{eq:h-local-expansion}, we obtain
\begin{equation*}
 \begin{aligned}
 &\int_{\min\{y_j,y_j(\widetilde\theta)\}}^
 {\max\{y_j,y_j(\widetilde\theta)\}}
 |h(y)|\dd y\\
 &\quad=
 \frac{p_Y(y_j;\theta)|\delta_y(y_j;\theta)|}{4}
 \{y_j(\widetilde\theta)-y_j\}^2
 +o\left(\{y_j(\widetilde\theta)-y_j\}^2\right)\\
 &\quad=
 \frac{p_Y(y_j;\theta)}{4|\delta_y(y_j;\theta)|}
 \left\{g(y_j;\theta)^\top
 (\widetilde\theta-\theta)\right\}^2
 +o(\|\widetilde\theta-\theta\|^2),
 \end{aligned}
\end{equation*}
where the second equality follows from \eqref{eq:root-gradient}.

Summing over the finitely many boundaries yields
\begin{equation}
 R(\widetilde\theta;\theta)-\err(\theta)
 =\frac12(\widetilde\theta-\theta)^\top
 H_R(\theta)(\widetilde\theta-\theta)
 +o(\|\widetilde\theta-\theta\|^2),
 \label{eq:risk-quadratic-expansion}
\end{equation}
where
\begin{equation*}
 H_R(\theta)
 =\sum_{j=1}^K
 \frac{p_Y(y_j;\theta)}{2|\delta_y(y_j;\theta)|}
 g(y_j;\theta)g(y_j;\theta)^\top.
\end{equation*}
This proves \eqref{eq:risk-hessian} and also shows that
\begin{equation*}
 \left.
 \frac{\partial R(\widetilde\theta;\theta)}
 {\partial\widetilde\theta}
 \right|_{\widetilde\theta=\theta}=0.
\end{equation*}

Now set
\begin{equation*}
 \Delta_n=\widehat\theta-\theta.
\end{equation*}
By root-$n$ consistency, $\Delta_n=O_p(n^{-1/2})$, and therefore
\eqref{eq:risk-quadratic-expansion} gives
\begin{equation}
 R(\widehat\theta;\theta)-\err(\theta)
 =\frac12\Delta_n^\top H_R(\theta)\Delta_n
 +o_p(n^{-1}).
 \label{eq:random-risk-expansion}
\end{equation}

To justify taking expectations in
\eqref{eq:random-risk-expansion}, fix a radius $r>0$ such that the
local expansion \eqref{eq:risk-quadratic-expansion} holds whenever
$\|\Delta\|\le r$, and write
\[
 \mathcal A_n=\{\|\Delta_n\|\le r\}.
\]
Define
\[
 \operatorname{rem}(\Delta)
 =
 R(\theta+\Delta;\theta)-\err(\theta)
 -\frac12\Delta^\top H_R(\theta)\Delta.
\]
By \eqref{eq:risk-quadratic-expansion},
\[
 \operatorname{rem}(\Delta)=o(\|\Delta\|^2)
 \qquad\text{as }\Delta\to0.
\]
Hence there exists a nonnegative deterministic function $\eta$,
with $\eta(t)\to0$ as $t\downarrow0$, such that
\[
 |\operatorname{rem}(\Delta)|
 \le
 \eta(\|\Delta\|)\|\Delta\|^2,
 \qquad \|\Delta\|\le r.
\]
Condition \eqref{eq:higher-moment-condition} implies that
$\{n\|\Delta_n\|^2\}$ is uniformly integrable, while root-$n$
consistency gives
$\eta(\|\Delta_n\|)\to0$ in probability. Hence
\[
 n\,\E\!\left[
 |\operatorname{rem}(\Delta_n)|
 \mathbf 1_{\mathcal A_n}
 \right]\longrightarrow0.
\]

Moreover, Markov's inequality and
\eqref{eq:higher-moment-condition} give
\[
 \Prb(\mathcal A_n^c)
 \le
 r^{-(2+\epsilon)}\E\|\Delta_n\|^{2+\epsilon}
 =
 O(n^{-1-\epsilon/2})
 =
 o(n^{-1}).
\]
Because
\[
 0\le R(\widehat\theta;\theta)-\err(\theta)\le1,
\]
the contribution of $\mathcal A_n^c$ to the expected excess risk is
$o(n^{-1})$. Uniform integrability of $n\|\Delta_n\|^2$ also gives
\[
 \E\!\left[
 \|\Delta_n\|^2\mathbf 1_{\mathcal A_n^c}
 \right]
 =o(n^{-1}),
\]
so the quadratic term may likewise be restricted to
$\mathcal A_n$ without changing its expectation at order $n^{-1}$.
Consequently,
\begin{equation*}
 \begin{aligned}
 \E\{R(\widehat\theta;\theta)\}-\err(\theta)
 &=\frac12
 \tr\left[
 H_R(\theta)
 \E\{\Delta_n\Delta_n^\top\}
 \right]
 +o(n^{-1})\\
 &=\frac{1}{2n}\tr\{H_R(\theta)V\}
 +o(n^{-1}),
 \end{aligned}
\end{equation*}
where the second equality follows from
\eqref{eq:moment-convergence}. Substitution of
\eqref{eq:risk-hessian} proves
\eqref{eq:excess-risk-generalV}. For a correctly specified regular
maximum-likelihood estimator, setting $V=\J(\theta)^{-1}$ gives
\eqref{eq:excess-risk}; for the ignoring estimator,
\Cref{lem:ig-information} gives the corresponding result with
$V=\J_{\mathrm{ig}}^{-1}$.
\end{proof}

\section{Hard-drive expert-label construction and diagnostics}
\label{app:hard-drive-experts}

A separate, fully labelled training subset was used to fit a committee of seven classifiers: logistic regression, linear discriminant analysis, regularised discriminant analysis, a support vector machine, random forest, $k$-nearest neighbours, and naive Bayes. To avoid information leakage, this training subset was excluded from the subsequent Weibull-mixture analysis. For observation $j$, we constructed the following committee-based ranking score:
\[
\kappa_j
=
a_j
+
\left|\overline p_{2j}-\tfrac12\right|
-
s_{2j},
\]
where $a_j$ denotes the proportion of classifiers agreeing on the predicted class, $\overline p_{2j}$ is the mean predicted probability of class~2 across the committee, and $s_{2j}$ is the corresponding between-classifier standard deviation. Thus, a larger value of $\kappa_j$ reflects stronger agreement, a mean prediction farther from the classification threshold of $1/2$, and less variation among the committee members. The score is used only to rank observations for the semi-synthetic label-missingness construction.

Observations were ranked in decreasing order of $\kappa_j$. Labels were then released for the highest-ranked 60\%, 70\%, and 80\% of observations to form the junior, intermediate, and senior expert regimes, respectively. This construction represents the idea that more experienced experts can assign labels to a larger proportion of ambiguous observations. The three labelled sets were nested, and the resulting missing-label fractions were 0.402, 0.304, and 0.201, respectively.

For expert regime $e\in\{\mathrm{J},\mathrm{I},\mathrm{S}\}$, let $M_j^{(e)}$ denote the missing-label indicator for observation $j$, and let $\widehat{\theta}_{\mathrm{CC}}$ denote the common-shape CC estimate obtained from the fully labelled analysis sample. We then fitted the descriptive diagnostic model
\[
\operatorname{logit}
\left\{
\Pr\!\left(M_j^{(e)}=1\mid Y_j\right)
\right\}
=
\beta_{0e}
+
\beta_{1e}
\delta_C\!\left(Y_j;\widehat{\theta}_{\mathrm{CC}}\right)^2.
\]
The estimated slopes were negative in all three regimes (Table~\ref{tab:hard-drive-missingness}), consistent with observations closer to the estimated Bayes decision boundary being more likely to remain unlabelled. These estimates are used only as descriptive diagnostics; their magnitudes are not directly comparable across regimes because the regimes have different overall missing-label fractions.

\begin{table}[!ht]
\centering
\caption{Descriptive association between squared log-posterior odds
and label missingness.}
\label{tab:hard-drive-missingness}
\begin{tabular}{lrr}
\toprule
Expert & Missing fraction & $\widehat{\beta}_{1e}$ \\
\midrule
Junior       & 0.402 & $-14.157$ \\
Intermediate & 0.304 & $-11.910$ \\
Senior       & 0.201 & $-32.957$ \\
\bottomrule
\end{tabular}
\end{table}

The class-specific cross-validation results are reported in Table~\ref{tab:hard-drive-class-rates}. The same folds were used for all three expert regimes, so the $\mathrm{CC}$ rates were identical across regimes. Relative to $\mathrm{ig}$, $\mathrm{full}$ increased class~2 sensitivity but reduced class~1 specificity. Its lower overall estimated error rate therefore did not arise from uniform improvements in both classes.

\begin{table}[!ht]
\centering
\caption{Mean class-specific rates from five repetitions of five-fold cross-validation.}
\label{tab:hard-drive-class-rates}
\small
\begin{tabular*}{\linewidth}{@{\extracolsep{\fill}}lrrrrrr@{}}
\toprule
& \multicolumn{3}{c}{Class 2 sensitivity}
& \multicolumn{3}{c}{Class 1 specificity} \\
\cmidrule(lr){2-4}
\cmidrule(lr){5-7}
Expert
& $\mathrm{CC}$
& $\mathrm{ig}$
& $\mathrm{full}$
& $\mathrm{CC}$
& $\mathrm{ig}$
& $\mathrm{full}$ \\
\midrule
Junior
& 0.633 & 0.355 & 0.735
& 0.514 & 0.708 & 0.485 \\
Intermediate
& 0.633 & 0.346 & 0.748
& 0.514 & 0.718 & 0.473 \\
Senior
& 0.633 & 0.412 & 0.721
& 0.514 & 0.685 & 0.487 \\
\bottomrule
\end{tabular*}
\end{table}

\FloatBarrier
\bibliographystyle{apalike}
\bibliography{weibull_ssl_refs}

\end{document}